\documentclass{article}
\usepackage{spconf,amsmath,graphicx,hyperref}

\usepackage{wasysym}
\usepackage[utf8]{inputenc} % allow utf-8 input
\usepackage[T1]{fontenc}    % use 8-bit T1 fonts
\usepackage{hyperref}  % hyperlinks
\usepackage{algorithm}
\usepackage[noend]{algpseudocode}
\usepackage{amsthm}
\usepackage{url}            % simple URL typesetting
\usepackage{booktabs}       % professional-quality tables
\usepackage{amsmath,amssymb}%
\usepackage{amsfonts}       % blackboard math symbols
\usepackage{nicefrac}       % compact symbols for 1/2, etc.
\usepackage{microtype}      % microtypography
\usepackage{xcolor}         % colors

\newtheorem{Proposition}{Proposition}

\usepackage{subfig}
\usepackage{tabularx}
\usepackage{forest}
\usepackage{fancyhdr}
\usetikzlibrary{arrows.meta}
\usepackage{todonotes}
\usepackage{arydshln}
\usepackage{tikz}
\usepackage{amsmath,amsfonts,bm}

\def\eqref#1{equation~\ref{#1}}
\def\Eqref#1{Equation~\ref{#1}}
\def\1{\bm{1}}

\def\ve{{\bm{e}}}

\def\vx{{\bm{x}}}

\def\mA{{\bm{A}}}

\def\mD{{\bm{D}}}

\def\mF{{\bm{F}}}

\def\mI{{\bm{I}}}

\def\mP{{\bm{P}}}

\def\mW{{\bm{W}}}

\DeclareMathAlphabet{\mathsfit}{\encodingdefault}{\sfdefault}{m}{sl}
\SetMathAlphabet{\mathsfit}{bold}{\encodingdefault}{\sfdefault}{bx}{n}

\title{DAG Structure Learning with Node-Varying Heavy-Tailed Disturbances}
\name{Robin Francis and Sundeep Prabhakar Chepuri
}
\address{Indian Institute of Science, Bangalore, India}
\begin{document}
\ninept
\maketitle
\begin{abstract}
DAG structures can be learned from observations by incorporating acyclicity via smooth constraints, such as DAGMA's log-determinant function. Typically, observations and the underlying DAG are related through linear structural equation models, with Gaussianity assumption on the disturbances at each node. Extending DAG learning methods that account for scale and tail behavior of disturbances that vary across nodes remains a challenge. We propose a node-varying Student-$t$ DAGMA that jointly estimates the graph structure, node-specific disturbance scales, and the degrees-of-freedom parameter. A Gaussian scale-mixture representation yields an iteratively reweighted formulation that down-weights large standardized residuals. Building on existing results, we also establish identifiability of the graph and disturbance parameters. Synthetic experiments demonstrate improved structural recovery under heterogeneous heavy-tailed disturbances, and experiments on real finance data show that the proposed method yields lower test negative log-likelihood than the evaluated baselines.
\end{abstract}
\begin{keywords}
Directed acyclic graphs, expectation-maximization, heavy-tailed noise, Student-$t$ distribution.
\end{keywords}
\section{Introduction}
Structural equation models (SEMs) describe relationships among variables and are widely used in causal inference, including applications in economics and genetics~\cite{pearl2009causality_SEM}. We consider linear SEMs associated with directed acyclic graphs (DAGs), in which directed edges represent direct causal effects under the model assumptions. Learning such graphs from data is challenging because acyclicity imposes a combinatorial constraint~\cite{chickering1996learning,Chickering_2004}. A further challenge arises when disturbances differ across nodes in both scale and tail behavior. For example, in financial networks, some assets may experience frequent extreme shocks, whereas others exhibit relatively stable fluctuations, as illustrated in Fig.~\ref{fig:illustration}. Accounting for heterogeneity requires learning the magnitude of disturbances from their likelihood that produce extreme observations, motivating DAG learning methods that adapt to both characteristics at the node level.

NOTEARS~\cite{NEURIPS2018_notears} introduced a differentiable characterization of acyclicity, paving the way for a growing body of work on continuous optimization for DAG structure learning. Later methods include GOLEM~\cite{NEURIPS2020_golem}, which uses a likelihood-based objective, and DAGMA~\cite{NEURIPS2022_dagma}, which combines a log-determinant acyclicity characterization with a path-following optimization strategy. In parallel, other works addresses more general disturbance models. For instance, LiNGAM exploits non-Gaussianity for causal discovery~\cite{JMLRLiNGAM}, while heavy-tailed causal models have been studied using extremal dependence measures~\cite{gnecco2020causaldiscoveryheavytailedmodels}. Student-$t$ distributions have also been used for robust graphical modeling through EM-based penalized likelihood estimation, but primarily for undirected graphs ~\cite{Robust_DAG_finegold}. Within continuous DAG learning, Huber-DAGMA replaces the squared loss with a Huber loss to reduce the influence of large residuals~\cite{PENG2026_Huber_DAGMA}, but does not explicitly estimate node-specific tail parameters. Heteroscedastic approaches allow noise variances to vary across variables or samples~\cite{duong2023heteroscedasticcausalstructurelearning,Yin_Heteroscedastic_2024}. In particular, CoLiDE jointly estimates the DAG structure and node-specific noise scales~\cite{saboksayr2024colide}. Although these methods address robustness or scale heterogeneity, adapting noise scales alone does not capture differences in tail behavior across nodes. {\color{black}Fig. 1(b) demonstrates the benefit of node-specific tail adaptation, with lower negative log-likelihood (NLL) under heterogeneous disturbances.} This motivates a continuous DAG learning framework that explicitly models both forms of disturbance heterogeneity.
\begin{figure}
    \captionsetup{font=scriptsize}
\centering
\subfloat[\label{fig:illus1}]{\includegraphics[width=0.4\columnwidth ]{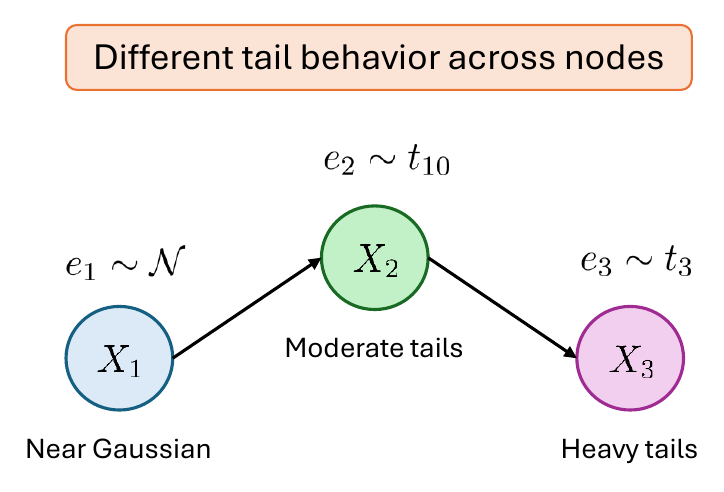}}
~
\subfloat[\label{fig:illus2}]{\includegraphics[width=0.4\columnwidth ]{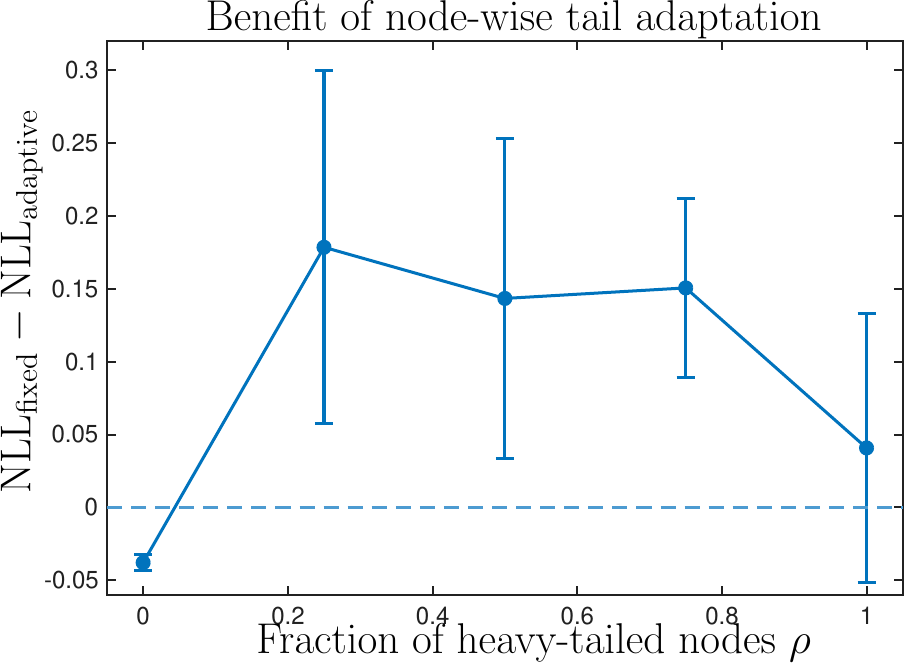}}
\vspace{-2mm}
\caption{ \footnotesize Illustration of node-varying heavy-tailed disturbances and the benefit of node-varying Student-$t$ DAGMA. $(a)$ Different nodes exhibit different disturbance characteristics. $(b)$ Difference in NLL between fixed and adaptive models as the fraction of heavy-tailed nodes varies.}
\vspace{-6mm}
\label{fig:illustration}
\end{figure}
In this work, we propose node-varying Student-$t$ DAGMA for learning linear DAGs under heterogeneous heavy-tailed disturbances. Our main contributions are:
\begin{itemize}
\item \emph{Node-specific scale and tail modeling.} We assign each node its own Student-$t$ scale and degrees-of-freedom (DoF) parameters, capturing differences in both disturbance magnitude and tail heaviness across nodes. Smaller DoF values accommodate heavier tails and downweight extreme residuals more strongly, while larger values approach Gaussian behavior.
\item \emph{Joint structure and disturbance estimation.} We integrate this model into DAGMA's log-determinant acyclicity formulation and central-path optimization strategy to jointly estimate the DAG structure, disturbance scales, and DoF parameters. This extends scale-adaptive continuous DAG learning to account explicitly for node-varying tail behavior.
\color{black}\item \emph{Identifiability analysis.} Building on established independent component analysis (ICA) and LiNGAM results, we establish identifiability of the graph structure and node-specific Student-$t$ parameters under the stated model assumptions.
\end{itemize}
{\color{black}Synthetic and financial-data experiments demonstrate improved DAG recovery and lower held-out NLL, respectively, compared with the evaluated baselines.}

\section{Preliminaries}

\noindent\textbf{Notation.} We use boldface lowercase and uppercase letters to denote vectors and matrices, respectively, and calligraphic letters to denote sets. For vectors, $\|\cdot\|_p$ denotes the $\ell_p$ norm, $p\geq 1$. For matrices, $\|\mW\|_1=\sum_{i,j}|W_{ij}|$ denotes the entrywise $\ell_1$ norm. $\mI$ denotes the $d\times d$ identity matrix, $\rho(\cdot)$ denotes the spectral radius of its matrix argument, and $\circ$ denotes the elementwise Hadamard product. The expectation operator is denoted by $\mathbb{E}$.

We begin by reviewing the DAGMA formulation~\cite{NEURIPS2022_dagma} and its least-squares specialization under Gaussian disturbances.

\noindent\textbf{DAGs and Linear SEMs.}
A DAG $\mathcal{G}=(\mathcal{V},\mathcal{E})$ consists of a node set $\mathcal{V}=\{1,\ldots,d\}$ and a directed edge set $\mathcal{E}\subseteq\mathcal{V}\times\mathcal{V}$, with no directed cycles. An edge $(i,j)\in\mathcal{E}$, denoted by $i\to j$, identifies $i$ as a parent of $j$ and $j$ as a child of $i$. The graph is represented by a weighted adjacency matrix $\mW\in\mathbb{R}^{d\times d}$, where $W_{ij}\neq 0$ if and only if $i\to j$, and $W_{jj}=0$. Every DAG admits a topological ordering in which all edges point from earlier to later nodes. Under this ordering, $\mW$ is strictly upper triangular and hence nilpotent, satisfying $\mW^d=\boldsymbol{0}$. For arbitrary signed weights, nilpotency of $\mW$ alone does not imply acyclicity because contributions from different walks can cancel. An exact characterization instead uses the nonnegative matrix $\mW\circ\mW$, i.e., the graph is acyclic if and only if $(\mW\circ\mW)^d=\boldsymbol{0}$, equivalently $\rho(\mW\circ\mW)=0$~\cite{NEURIPS2022_dagma}.

Consider $N$ independent observations 
$\{\vx_m\}_{m=1}^{N}$ with each observation $\vx_m\in\mathbb{R}^{d}$ generated by the linear SEM
\begin{equation}\label{eq:Linear_SEM}
\vx_m=(\mW)^{\mathsf T}\vx_m+\ve_m.
\end{equation}
We assume that the structural disturbance vectors $\ve_m$ are identically distributed and independent across observations with mutually independent components across nodes.
Specifically, $e_{mj}=x_{mj}-\sum_{i\neq j}W_{ij}x_{mi}$ is the structural disturbance at node $j$.
\\

\noindent\textbf{DAGMA.} For $s>0$, define the domain of feasible matrices, $\mathcal{D}_s= \left\{\mW\in\mathbb{R}^{d\times d}:  \rho(\mW\circ\mW)<s\right\}. $
DAGMA uses the log-determinant acyclicity function~\cite{NEURIPS2022_dagma}:
\begin{equation}\label{eq:dagma_h}
h_s(\mW)=
-\log\det\!\left(s\mI-\mW\circ\mW\right)+d\log s.
\end{equation}
For $\mW\in\mathcal{D}_s$, this function is nonnegative and satisfies
\begin{equation*}
h_s(\mW)=0
\ \Longleftrightarrow
\mW\text{ represents a DAG}
\ \Longleftrightarrow
\nabla h_s(\mW)=0.
\end{equation*}
Thus, $h_s$ provides a smooth penalty that vanishes for acyclic graphs. DAGMA estimates the adjacency matrix $\mW$ through a sequence of optimization problems
\begin{equation}\label{eq:DAGMA_prob}
\text{minimize}
F_{\mu_k}(\mW)
:=
\mu_k\left(L(\mW)+\lambda\|\mW\|_1\right)+h_s(\mW),
\end{equation}
over $\mW\in\mathcal{D}_s, W_{jj}=0,j=1,\ldots,d$. Here,  $L(\mW)$ is the data-fitting loss, $\lambda$ controls sparsity, and $\mu_k$ is gradually decreased over the iterations, increasing the relative weight of the acyclicity-promoting term. For independent Gaussian disturbances with a fixed common variance, i.e., $\ve_m\sim\mathcal{N}(\boldsymbol{0},\sigma^2\mI)$, the NLL on the DAG model is proportional, up to terms independent of $\mW$, to the least-squares penalty, $ L(\mW)=
\frac{1}{N}\sum_{m=1}^{N}\|(\mI - \mW^{\mathsf T})\vx_m\|_2^2.$ 
The quadratic dependence on the residual magnitude in the penalty term results in a large penalty for extreme residuals, motivating an alternative loss function for heavy-tailed disturbances.

\section{Problem Modeling and Formulation}

We consider independent node-wise Student-$t$ disturbances,
$e_{mj} \sim t_{\nu_j}(0,\sigma_j),$ where $\sigma_j>0$ and
$\nu_j>0$ denote the scale and DoF parameters of node $j$, respectively. The Student-$t$ model
accommodates the heavy-tailed disturbances and extreme observations (common in financial data). Different variables may exhibit distinct disturbance
characteristics. For instance, the scale parameter $\sigma_j$ captures
disturbance magnitude, while $\nu_j$ controls tail behavior.
We therefore allow both parameters to vary across nodes
to account for heterogeneity in the underlying disturbance distributions.

Under the assumed disturbance model, the normalized NLL is
\begin{equation}
\begin{aligned}
\label{eq:student_nll}
L(\theta) =&\frac{1}{N}\sum_{m=1}^{N}\sum_{j=1}^{d}
\log\sigma_j+c(\nu_j){+}\frac{\nu_j{+}1}{2} 
\log\left(1+\frac{e_{mj}^2}{\nu_j\sigma_j^2}\right),
\end{aligned}
\end{equation}
where $\theta = \{\mW,\mathbf{\sigma},\mathbf{\nu}\} $ and $c(\nu)=\frac{1}{2}\log(\nu\pi)+\log\Gamma(\nu/2)
-\log\Gamma((\nu+1)/2)$.
{\color{black}
We aim to jointly estimate the weighted adjacency matrix as well as the scale and DoF parameters $\{\mathbf W, \boldsymbol{\sigma},\boldsymbol{\nu}\},$ by minimizing the Student-$t$ NLL subject to graph sparsity and acyclicity.} The joint estimation of the DAG structure and node-specific disturbance parameters is formulated as
\begin{equation}
\label{eq:joint_problem}
\begin{aligned}
(\mathcal{P})\quad
&\underset{\mW,\boldsymbol{\sigma},\boldsymbol{\nu}}{\text{minimize}}
\quad  L(\mW,\boldsymbol{\sigma},\boldsymbol{\nu})+\lambda\|\mW\|_1\\
&\text{s. to}\quad \quad h_s(\mW)=0,\; \rho(\mW\circ\mW)<s,\, \sigma_j,\nu_j>0, 
\end{aligned}
\end{equation}
where $s>0$, $L(\mW,\boldsymbol{\sigma},\boldsymbol{\nu})$ is defined in \Eqref{eq:student_nll}, and $j=1,\ldots,d$. The objective balances data fit and graph sparsity, while the constraints enforce acyclicity and valid disturbance parameters. {\color{black}A direct approach to $(\mathcal{P})$ requires handling a non-quadratic Student-$t$ loss, we bypass this to keep DAGMA's least-squares framework. By writing the Student-$t$ distribution as a Gaussian scale-mixture, in what follows, we can  an expectation-maximization (EM) based iteratively reweighted least-squares (IRLS) algorithm.}

\section{Node-varying Student-$\MakeLowercase{t}$ DAGMA Algorithm}
In this section, we develop our approach to solving $(\mathcal{P})$. Using the Gaussian scale-mixture representation of the Student-$t$ distribution, we derive a weighted least-squares surrogate that allow us to alternately update the graph structure, node specific scales, and DoF. We embed these updates in DAGMA's central-path framework, which gradually enforces acyclicity.

The Student-$t$ distribution admits the Gaussian scale-mixture representation:
$e_{mj}\mid\tau_{mj}
\sim\mathcal{N}\left(0,\frac{\sigma_j^2}{\tau_{mj}}\right),
\quad
\tau_{mj}\sim\operatorname{Gamma}
\left(\frac{\nu_j}{2},\frac{\nu_j}{2}\right)$, where $\tau_{mj}$ is a latent precision variable, and the
Gamma distribution is parameterized by shape and rate.
Marginalizing over $\tau_{mj}$ recovers the Student-$t$ distribution. Our objective is to jointly estimate the weighted adjacency
matrix $\mW$ and the node-specific parameters
$\boldsymbol{\sigma}$ and $\boldsymbol{\nu}$. 

To jointly estimate these parameters within the DAGMA framework of~\Eqref{eq:DAGMA_prob}, using the Student-$t$ NLL in~\Eqref{eq:student_nll}, we leverage the Gaussian scale-mixture representation through an EM procedure. Given the current parameter estimates $\theta$, the posterior distribution of the latent precision is, $
\tau_{mj}\mid e_{mj},\theta
\sim\operatorname{Gamma}\left(
\frac{\nu_j+1}{2},
\frac{\nu_j+e_{mj}^2/\sigma_j^2}{2}
\right).$ 
Consequently, the observation- and node-dependent
reliability weights are
\begin{equation}
\label{eq:reliability_weights}
a_{mj}
=\mathbb{E}[\tau_{mj}\mid e_{mj},\theta]
=\frac{\nu_j+1}{\nu_j+e_{mj}^2/\sigma_j^2}.
\end{equation}
Observations with large standardized residuals receive smaller weights, thereby reducing their influence on the subsequent graph update. {\color{black}The EM formulation replaces the non-quadratic Student-$t$ loss with a quadratic surrogate. Specifically, taking the conditional expectation of the complete-data NLL with respect to the latent precisions yields a weighted least-squares objective, where the weights $a_{mj}^{(\ell)}$ are computed in the E-step and held fixed during the graph update. Consequently, the graph-dependent component of the EM surrogate is given by,} $J(\mW,\mA)=\frac{1}{2N}
\sum_{m=1}^{N}\sum_{j=1}^{d}
\frac{a_{mj}}{\sigma_j^2}e_{mj}^2$. Let $\mA=[a_{mj}]\in\mathbb{R}^{N\times d}$ and define
$\mD_m=\operatorname{diag}\left(
\frac{a_{m1}}{\sigma_1^2},\ldots,
\frac{a_{md}}{\sigma_d^2}\right).$
With the weights and current scale estimates held fixed, the graph-dependent component of the EM surrogate is
\begin{equation}
\label{eq:weighted_loss}
\begin{aligned}
J(\mW,\mA)&=\frac{1}{2N}\sum_{m=1}^{N}\left\|\mD_m^{1/2}(\mI-\mW^{\mathsf T})\vx_m\right\|_2^2.
\end{aligned}
\end{equation}
Each diagonal entry of $\mD_m$ weights the corresponding
node residual according to its reliability and disturbance
scale. The dependence of $J$ on the fixed scale estimates is suppressed for brevity.
\begin{algorithm}[t]
\caption{Node-varying Student-$t$ DAGMA}
\label{alg:tDAGMA}
\begin{algorithmic}[1]
\State Initialize $\mW_0$ with zero diagonal and
$\rho(\mW_0\circ\mW_0)<s$.
\State Initialize $\mu_0,\sigma_{j,0}>0$,
$\nu_{j,0}\in[\nu_{\min},\nu_{\max}]$, and
$j=1,\ldots,d$.
\For{$k=0,1,\ldots,K-1$}
\State Set $\mW^{(0)}\gets\mW_k$,
$\boldsymbol{\sigma}^{(0)}\gets\boldsymbol{\sigma}_k$,
and $\boldsymbol{\nu}^{(0)}\gets\boldsymbol{\nu}_k$.
\For{$\ell = 0,\ldots \text{until convergence}$}
\State Compute
$e_{mj}^{(\ell)}
=x_{mj}-\sum_{i\neq j}W_{ij}^{(\ell)}x_{mi}$,
$\forall m,j$.
\State Evaluate $a_{mj}^{(\ell)}$ and
$b_{mj}^{(\ell)}$ using eq.~\ref{eq:reliability_weights} and~\ref{eq:bj_equation}, respectively.
\State Approximately solve eq.~\ref{eq:W_update}
to obtain $\mW^{(\ell+1)}$.
\State Recompute
$e_{mj}^{(\ell+1)}
=x_{mj}-\sum_{i\neq j}W_{ij}^{(\ell+1)}x_{mi}$,
$\forall m,j$.
\For{$j=1,\ldots,d$}
\State Update $\sigma_j^{(\ell+1)}$
using eq.~\ref{eq:sigma_update}.
% \State Compute $\bar a_j^{(\ell)}$
% and $\bar b_j^{(\ell)}$.
\State Update $\nu_j^{(\ell+1)}$
using eq.~\ref{eq:nu_update} and acceptance criterion.
\EndFor
\EndFor
\State Set
$(\mW_{k+1},\boldsymbol{\sigma}_{k+1},
\boldsymbol{\nu}_{k+1})
\gets
(\mW^{(\ell)},\boldsymbol{\sigma}^{(\ell)},
\boldsymbol{\nu}^{(\ell)})$.
\State Set $\mu_{k+1}\gets\eta\mu_k$,
where $0<\eta<1$.
\EndFor
\State \Return $\widehat{\mW}=\mW_K$,
$\widehat{\boldsymbol{\sigma}}=\boldsymbol{\sigma}_K$,
$\widehat{\boldsymbol{\nu}}=\boldsymbol{\nu}_K$.
\end{algorithmic}
\end{algorithm}
The node-specific scale and DoF parameters are subsequently updated within the EM procedure.

The graph update is obtained by replacing the data-fitting term in the DAGMA objective with \Eqref{eq:weighted_loss}, yielding an IRLS procedure. The node-specific scale and DoF parameters are subsequently updated within the EM procedure.

We solve this problem using a nested optimization procedure with an outer DAGMA central-path iteration, indexed by $k$, and an inner generalized EM iteration, indexed by $\ell$. For each fixed path parameter $\mu_k$, the inner loop alternately updates the graph matrix $\mW$ and the node-specific parameters $\boldsymbol{\sigma}$ and $\boldsymbol{\nu}$. Upon convergence, $\mu_k$ is decreased to progressively enforce acyclicity.
\\

\noindent\textbf{Iteratively Reweighted DAGMA.} At each inner iteration $\ell$, the E-step evaluates the residuals
$e_{mj}^{(\ell)}=x_{mj}-\sum_{i\neq j}W_{ij}^{(\ell)}x_{mi}$ and computes the corresponding reliability weights $a_{mj}^{(\ell)}$ using \Eqref{eq:reliability_weights}. Let $\mA^{(\ell)}=[a_{mj}^{(\ell)}]$. These weights are held fixed during the subsequent parameter updates, yielding a weighted DAGMA subproblem for estimating the graph structure.

Given these weights and the current scale estimates $\boldsymbol{\sigma}^{(\ell)}$, the graph update $\mW^{(\ell+1)}$ is obtained by approximately solving
\begin{align}\label{eq:W_update}
\mW^{(\ell + 1)}=\arg\min_{\mW} \,\mu_k \left[ J(\mW,\mA^{(\ell)}) {+}\lambda\|\mW\|_1\right] {+}h_s(\mW).
\end{align}
The scale-mixture representation of the Student-$t$ likelihood transforms the standard least-squares DAGMA step into an iteratively reweighted DAGMA problem. The weighted loss provides robust structural fitting, the $\ell_1$-regularizer promotes sparsity, and $h_s(\mW)$ penalizes cyclic structure. \textcolor{black}{In the implementation, the $\mW$-subproblem is solved using the proximal gradient method with backtracking line search, and updates that do not increase the penalized EM surrogate are only accepted.}

\noindent\textbf{Scale and DoF update.}
Using the updated graph matrix $\mW^{(\ell+1)}$, we recompute the residuals as
$e_{mj}^{(\ell+1)}
=x_{mj}-\sum_{i\neq j}W_{ij}^{(\ell+1)}x_{mi}$.
Minimizing the EM surrogate with respect to the squared scale parameter gives
\begin{equation}\label{eq:sigma_update}
\left(\sigma_j^{(\ell+1)}\right)^2
=\frac{1}{N}\sum_{m=1}^{N}
a_{mj}^{(\ell)}
\left(e_{mj}^{(\ell+1)}\right)^2.
\end{equation}
Here, $a_{mj}^{(\ell)}$ is the posterior mean of the latent precision, computed in the E-step. Thus, the squared scale is updated using reliability-weighted squared residuals from the updated graph. The algorithm therefore does not require the disturbance scale to be known in advance and adapts it separately for each structural equation.

Next, we estimate the DoF parameter $\nu_j$, which controls the heaviness of the Student-$t$ tails. Define the E-step averages
$\bar a_j^{(\ell)}
=\frac{1}{N}\sum_{m=1}^{N}a_{mj}^{(\ell)},
\qquad
\bar b_j^{(\ell)}
=\frac{1}{N}\sum_{m=1}^{N}b_{mj}^{(\ell)},$
where $b_{mj}^{(\ell)}=\mathbb{E}[\log\tau_{mj}\mid e_{mj}^{(\ell)},\theta^{(\ell)}]$ and $\theta^{(\ell)} =(\mW^{(\ell)},\boldsymbol{\sigma}^{(\ell)}, \boldsymbol{\nu}^{(\ell)})$. The posterior expectation is
\begin{equation}\label{eq:bj_equation}
\begin{aligned}
b_{mj}^{(\ell)}
&=\psi_0(\frac{\nu_j^{(\ell)}+1}{2}) -\log[
\frac{\nu_j^{(\ell)}
+{e_{mj}^{(\ell)}}^2/
{\sigma_j^{(\ell)}}^2}{2}],
\end{aligned}
\end{equation}
where $\psi_0(\cdot)$ denotes the digamma function.
The updated value $\nu_j^{(\ell+1)}$ is obtained by solving the following scalar equation for $\nu_j$,
\begin{equation}\label{eq:nu_update}
\log\left(\frac{\nu_j}{2}\right)+1
-\psi_0\left(\frac{\nu_j}{2}\right)
+\bar b_j^{(\ell)}-\bar a_j^{(\ell)}=0.
\end{equation}

\begin{figure*}[!t]
    \centering
    \includegraphics[width=0.95\linewidth]{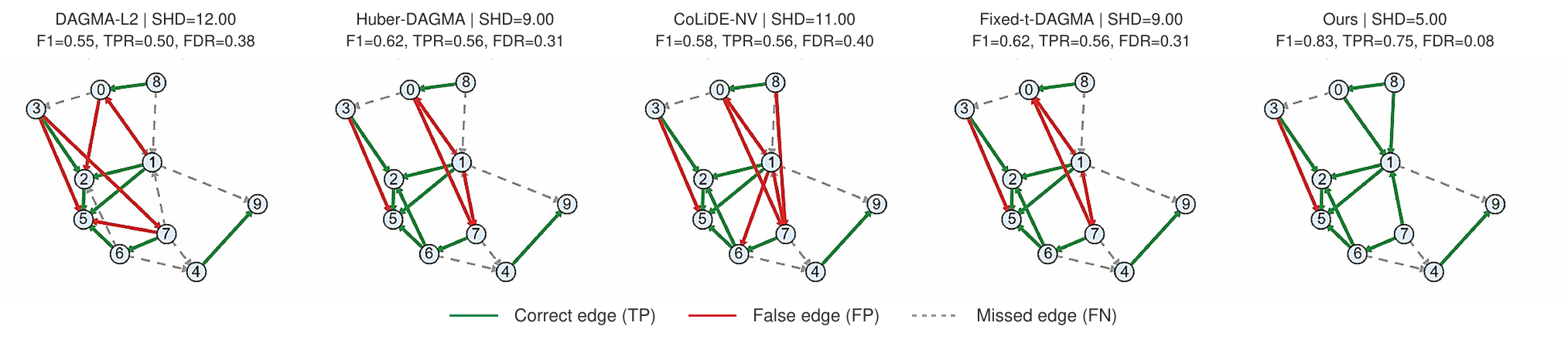}
    \vspace{-2mm}
    \caption{\footnotesize Recovered DAG structures obtained using the competing methods. TP, FP, and FN denote true positive, false positive, and false negative, respectively.}
    %\vspace{-2mm}
    \label{fig:performance_syn_data}
\end{figure*}
The update of $\nu_j$ can be efficiently computed using bisection over $[\nu_{\min},\nu_{\max}]$, where $0<\nu_{\min}<\nu_{\max}$. Only updates that do not increase the EM surrogate are accepted to preserve monotonicity. A smaller $\nu_j$ corresponds to heavier tails and stronger downweighting of large standardized residuals, whereas the reliability weights approach one as the DoF tends to infinity, recovering the Gaussian model. Thus, the proposed method adapts the disturbance scale and robustness separately for each node.

For each fixed $\mu_k$, the inner EM iterations alternately update $\mW$, $\boldsymbol{\sigma}$, and $\boldsymbol{\nu}$ until convergence, with the latent posterior expectations $a_{mj}^{(\ell)}$ and $b_{mj}^{(\ell)}$ computed in the E-step and held fixed during the subsequent parameter updates. The resulting estimates warm-start the next outer iteration, where $\mu_{k+1}=\eta\mu_k$, $0<\eta<1$. Progressively decreasing $\mu_k$ places greater emphasis on acyclicity, driving the solution toward the DAG set. The overall procedure is summarized in Algorithm~\ref{alg:tDAGMA}.

\section{Identifiability}
In this section, we establish the identifiability of the proposed Student-$t$ DAG model by applying standard ICA and LiNGAM identifiability arguments~\cite{comon1994ICA,hyvarinen2013ICA,JMLRLiNGAM}.

The linear SEM can be expressed as $\vx_m=\mF\ve_m$, where $\mF=(\mI-\mW^{\mathsf T})^{-1}$. Since the structural disturbances are mutually independent, this representation corresponds to an ICA model. A fundamental ICA identifiability result~\cite{comon1994ICA,hyvarinen2013ICA} states that, for nondegenerate independent components with at most one Gaussian component, the mixing matrix is identifiable up to column permutation and nonzero scaling. Thus, the equivalent mixing matrix must satisfy $\widetilde{\mF}=\mF\mP\mD$, the matrices $\mP$ and $\mD$ denote a permutation matrix and an invertible diagonal matrix, respectively. In our case, the Student-$t$ disturbances are non-Gaussian for finite $\nu_j>2$. Using this property together with linearity, acyclicity, and absence of latent confounders, gives the following result.

\begin{Proposition}
Consider the linear SEM in~\Eqref{eq:Linear_SEM}, where $\mW$ represents a DAG and the structural disturbances are mutually independent with $e_j\sim t_{\nu_j}(0,\sigma_j)$, $\sigma_j>0$, and $2<\nu_j<\nu_{\max}<\infty$, for $j=1,\ldots,d$. Assume that all variables are observed and that there are no latent confounders. Then the model parameters $(\mW,\boldsymbol{\sigma},\boldsymbol{\nu})$ are identifiable from the observational distribution.
\end{Proposition}

\begin{proof}
Since $\mW$ represents a DAG, the SEM admits the ICA
representation $\vx=\mF\ve$, where the disturbances are mutually independent and non-Gaussian for finite $\nu_j$. By the ICA identifiability result, $\mF$ is identifiable up to column permutation and nonzero scaling. For any observationally equivalent parameterization, we therefore have, $\widetilde{\mF}^{-1}
=\mD^{-1}\mP^{\mathsf T}\mF^{-1}.$
The acyclicity condition ensures that $\mW$ is similar to a strictly triangular matrix, and the absence of self-loops guarantees that $\mF^{-1}=\mI-\mW^{\mathsf T}$ has a unit diagonal. Hence, $\mF^{-1}$ and $\widetilde{\mF}^{-1}$ must have unit diagonal and $\mP^{\mathsf T}\mF^{-1}$ must have no zero diagonal entries. Under the acyclicity assumption~\cite{JMLRLiNGAM}, any row permutation that produces a nonzero diagonal entry would indicate a directed cycle in $\mW$, which leads to a contradiction, hence $\mP=\mI$. Moreover, $\mD=\mI$ resolves the scaling issue. Therefore, $\mW=\mI-\mF^{-\mathsf T}$ is uniquely identifiable. The disturbance distribution is then uniquely determined by $\ve=(\mI-\mW^{\mathsf T})\vx$. By the identifiability of the univariate Student-$t$ family, the node-specific parameters $\sigma_j$ and $\nu_j$ are therefore uniquely identified.
\end{proof}

\section{Numerical Experiments}
In this section, we evaluate the proposed method on synthetic and real-world datasets, assessing DAG recovery on synthetic data and held-out performance on financial data. We compare its performance against DAGMA~\cite{NEURIPS2022_dagma}, Huber-DAGMA~\cite{PENG2026_Huber_DAGMA}, CoLiDE-NV~\cite{saboksayr2024colide}, and Fixed-$t$-DAGMA, a variant of our method with a fixed DoF parameter~$\nu$. For all methods, we set the sparsity parameter $\lambda=0.1$ and employ the DAGMA central path $(\mu_0,\mu_1,\mu_2,\mu_3) =(1,0.1,0.01,0.001)$ with $s=1.5$. We allow at most $15$ EM iterations where applicable and $150$ inner iterations.

\vspace{2pt}
\noindent\textbf{Synthetic Data.} We generate synthetic data from a linear SEM with an underlying Erd\H{o}s--R\'enyi DAG with $d=10$ nodes and an average node degree of $2$ (ER2). We generate $N=500$ observations with independent, heterogeneous Student-$t$ disturbances, where the node-specific parameters are selected from $\nu_j\in\{3,5,10\}$ and $\sigma_j\in \{1,6\}$. For the proposed method, we initialize each $\nu_j$ at $10$ and constrain its estimates to $[2.1,50]$. The Huber threshold and edge weight threshold are set to $1.3$ and $0.44$, respectively.

We evaluate graph recovery using structural Hamming distance (SHD), true positive rate (TPR), false discovery rate (FDR), and F1-score. SHD measures the structural discrepancy between the estimated and true DAGs. TPR is the fraction of true edges recovered, whereas FDR is the fraction of estimated edges that are incorrect. The F1-score is the harmonic mean of precision and recall. Lower SHD and FDR and higher TPR and F1-score indicate better recovery. The results are shown in Fig.~\ref{fig:performance_syn_data}.
\vspace{2pt}

\noindent\textbf{49 Industry Portfolios Dataset.} We evaluate the performance of the proposed method using the Kenneth French 49 Industry Portfolios daily-return dataset~\cite{french_data_2026}, which contains value-weighted returns for $49$ U.S. industry portfolios. We consider the period from January 2000 to December 2019. The data are divided chronologically into training, validation, and test sets containing $60\%$, $20\%$, and $20\%$ of the observations, respectively.

\begin{table}[h]
    \centering
    \small
    \setlength{\tabcolsep}{4pt}
    \resizebox{0.75\columnwidth}{!}{%
    \begin{tabular}{|l|c|c|c|}
        \hline
        Method & Test NLL & Test NRMSE & Edge count \\
        \hline
        DAGMA & $32.84$ & $0.469$ & $14$ \\
        \hline
        Huber-DAGMA & $31.61$ & $\mathbf{0.464}$ & $18$ \\
        \hline
        CoLiDE-NV & $32.17$ & $0.467$ & $23$ \\
        \hline
        Ours & $\mathbf{26.40}$ & $0.465$ & $15$ \\
        \hline
    \end{tabular}%
    }\vspace{-2mm}
    \caption{\footnotesize Test-set performance and
    recovered edge counts on the 49 Industry Portfolios
    dataset. Bold values indicate the lowest NLL and NRMSE.}
    \vspace{-2mm}
    \label{tab:real_data}
\end{table}
We evaluate performance using the normalized root-mean-square error (NRMSE) and NLL on the test set, as well as the final edge count of the recovered DAGs. The results are presented in Table~\ref{tab:real_data}. NRMSE measures reconstruction error, NLL assesses the fitted probabilistic model on held-out observations, and the edge count describes graph sparsity. The proposed method achieves the lowest reported test NLL, with a test NRMSE of $0.465$, compared with $0.464$ for Huber-DAGMA. It recovers $15$ edges, fewer than Huber-DAGMA ($18$) and CoLiDE-NV ($23$), but one more than DAGMA ($14$).

\section{Conclusions}
\vspace{-2mm}
We proposed node-varying Student-$t$ DAGMA for recovering DAGs with heterogeneous node-wise disturbances. Using the Gaussian scale-mixture representation of the Student-$t$ distribution, we formulated an IRLS problem and jointly estimated the graph structure, disturbance scales, and DoF parameters. Numerical experiments demonstrated improved DAG recovery under heterogeneous heavy-tailed disturbances in the synthetic setting considered. On real-world financial data, the proposed method achieved the lowest reported test NLL, with reconstruction error comparable to the competing methods, while recovering a sparse graph.

\clearpage
\section{Acknowledgements}
This work is supported by the KIAC Fintech grant. Beyond this support, the authors have no relevant financial or non-financial interests to disclose. Claude AI was used to assist in writing the manuscript and coding the simulations. The authors take full responsibility for the results of this paper.

\section{Compliance with Ethical Standards}
This is a numerical simulation study for which no ethical approval was required.
\bibliographystyle{ieeetr}
\bibliography{strings,refs}

@inproceedings{NEURIPS2018_notears,
 author = {Zheng, Xun and Aragam, Bryon and Ravikumar, Pradeep K and Xing, Eric},
 booktitle = {Advances in Neural Information Processing Systems},
 editor = {S. Bengio and H. Wallach and H. Larochelle and K. Grauman and N. Cesa-Bianchi and R. Garnett},
 pages = {},
 publisher = {Curran Associates, Inc.},
 title = {DAGs with NO TEARS: Continuous Optimization for Structure Learning},
 url = {https://proceedings.neurips.cc/paper_files/paper/2018/file/e347c51419ffb23ca3fd5050202f9c3d-Paper.pdf},
 volume = {31},
 year = {2018}
}

@inproceedings{NEURIPS2020_golem,
 author = {Ng, Ignavier and Ghassami, AmirEmad and Zhang, Kun},
 booktitle = {Advances in Neural Information Processing Systems},
 editor = {H. Larochelle and M. Ranzato and R. Hadsell and M.F. Balcan and H. Lin},
 pages = {17943--17954},
 publisher = {Curran Associates, Inc.},
 title = {On the Role of Sparsity and DAG Constraints for Learning Linear DAGs},
 url = {https://proceedings.neurips.cc/paper_files/paper/2020/file/d04d42cdf14579cd294e5079e0745411-Paper.pdf},
 volume = {33},
 year = {2020}
}

@inproceedings{NEURIPS2022_dagma,
 author = {Bello, Kevin and Aragam, Bryon and Ravikumar, Pradeep},
 booktitle = {Advances in Neural Information Processing Systems},
 doi = {10.52202/068431-0598},
 editor = {S. Koyejo and S. Mohamed and A. Agarwal and D. Belgrave and K. Cho and A. Oh},
 pages = {8226--8239},
 publisher = {Curran Associates, Inc.},
 title = {DAGMA: Learning DAGs via M-matrices and a Log-Determinant Acyclicity Characterization},
 url = {https://proceedings.neurips.cc/paper_files/paper/2022/file/36e2967f87c3362e37cf988781a887ad-Paper-Conference.pdf},
 volume = {35},
 year = {2022}
}

@article{JMLRLiNGAM,
  author  = {Shohei Shimizu and Patrik O. Hoyer and Aapo Hyv  and rinen and Antti Kerminen},
  title   = {A Linear Non-Gaussian Acyclic Model for Causal Discovery},
  journal = {Journal of Machine Learning Research},
  year    = {2006},
  volume  = {7},
  number  = {72},
  pages   = {2003--2030},
  url     = {http://jmlr.org/papers/v7/shimizu06a.html}
}

@inproceedings{Robust_DAG_finegold,
title = "Robust graphical modeling with t-distributions",
author = "Finegold, \{Michael A.\} and Mathias Drton",
year = "2009",
language = "English",
series = "Proceedings of the 25th Conference on Uncertainty in Artificial Intelligence, UAI 2009",
publisher = "AUAI Press",
pages = "169--176",
booktitle = "Proceedings of the 25th Conference on Uncertainty in Artificial Intelligence, UAI 2009",
}

@misc{gnecco2020causaldiscoveryheavytailedmodels,
      title={Causal discovery in heavy-tailed models}, 
      author={Nicola Gnecco and Nicolai Meinshausen and Jonas Peters and Sebastian Engelke},
      year={2020},
      eprint={1908.05097},
      archivePrefix={arXiv},
      primaryClass={stat.ME},
      url={https://arxiv.org/abs/1908.05097}, 
}

@misc{duong2023heteroscedasticcausalstructurelearning,
      title={Heteroscedastic Causal Structure Learning}, 
      author={Bao Duong and Thin Nguyen},
      year={2023},
      eprint={2307.07973},
      archivePrefix={arXiv},
      primaryClass={cs.LG},
      url={https://arxiv.org/abs/2307.07973}, 
}

@inproceedings{Yin_Heteroscedastic_2024,
author = {Yin, Naiyu and Gao, Tian and Yu, Yue and Ji, Qiang},
title = {Effective causal discovery under identifiable heteroscedastic noise model},
year = {2024},
isbn = {978-1-57735-887-9},
publisher = {AAAI Press},
url = {https://doi.org/10.1609/aaai.v38i15.29586},
doi = {10.1609/aaai.v38i15.29586},
booktitle = {Proceedings of the Thirty-Eighth AAAI Conference on Artificial Intelligence and Thirty-Sixth Conference on Innovative Applications of Artificial Intelligence and Fourteenth Symposium on Educational Advances in Artificial Intelligence},
articleno = {1838},
numpages = {9},
series = {AAAI'24/IAAI'24/EAAI'24}
}

@book{pearl2009causality_SEM,
  title={Causality},
  author={Pearl, Judea},
  year={2009},
  publisher={Cambridge university press}
}

@article{Chickering_2004,
author = {Chickering, David Maxwell and Heckerman, David and Meek, Christopher},
title = {Large-Sample Learning of Bayesian Networks is NP-Hard},
year = {2004},
issue_date = {12/1/2004},
publisher = {JMLR.org},
volume = {5},
issn = {1532-4435},
journal = {J. Mach. Learn. Res.},
month = dec,
pages = {1287–1330},
numpages = {44}
}

@incollection{chickering1996learning,
  title={Learning Bayesian networks is NP-complete},
  author={Chickering, David Maxwell},
  booktitle={Learning from data: Artificial intelligence and statistics V},
  pages={121--130},
  year={1996},
  publisher={Springer}
}

@inproceedings{saboksayr2024colide,
  title={{CoLiDE: Concomitant Linear DAG Estimation}},
  author={Saboksayr, Seyed Saman and Mateos, Gonzalo and Tepper, Mariano},
  booktitle={International Conference on Learning Representations},
  year={2024}
}

@article{PENG2026_Huber_DAGMA,
title = {Robust causal representation learning for intelligent vehicle perception via adaptive huber-DAGMA framework},
journal = {Advanced Engineering Informatics},
volume = {74},
pages = {104802},
year = {2026},
issn = {1474-0346},
doi = {https://doi.org/10.1016/j.aei.2026.104802},
url = {https://www.sciencedirect.com/science/article/pii/S1474034626004945},
author = {Ziran Peng and Wang Li}}

@article{comon1994ICA,
  title={Independent component analysis, a new concept?},
  author={Comon, Pierre},
  journal={Signal processing},
  volume={36},
  number={3},
  pages={287--314},
  year={1994},
  publisher={Elsevier}
}

@article{hyvarinen2013ICA,
  title={Independent component analysis: recent advances},
  author={Hyv{\"a}rinen, Aapo},
  journal={Philosophical transactions. Series A, Mathematical, physical, and engineering sciences},
  volume={371},
  number={1984},
  pages={20110534},
  year={2013}
}

@misc{french_data_2026,
  author       = {French, Kenneth R.},
  title        = {Fama/{F}rench 49 Industry Portfolios},
  howpublished = {Kenneth R. French Data Library},
  organization = {Tuck School of Business, Dartmouth College},
  year         = {2026}, 
  note         = {Accessed: 2026-09-22},
  url          = {https://dartmouth.edu}
}
\clearpage

% \section{Appendix}

\end{document}